\documentclass{article}
\usepackage[T1]{fontenc}
\usepackage{amsmath}
\usepackage{amsthm}
\usepackage{float}
\usepackage{amsfonts}
\usepackage{blkarray}
\usepackage{graphicx}
\usepackage{epstopdf}
\usepackage{xcolor}
\usepackage{soul}
\usepackage{natbib}
\usepackage{cancel}
\usepackage{subfig}
\DeclareCaptionType{copyrightbox}
\usepackage[boxruled]{algorithm2e}
\def\E{\mathbb{E}}
\def\P{\mathbb{P}}
\def\1{\mathbf{1}}

\newtheorem{definition}{Definition}
\newtheorem{corollary}{Corollary}
\newtheorem{theorem}{Theorem}
\newtheorem{lemma}{Lemma}
\begin{document}
\title{Sparse Dueling Bandits}

\author{Kevin Jamieson, Sumeet Katariya\footnote{The first two authors are listed in alphabetical order as both contributed equally.}, Atul Deshpande and Robert Nowak \\
Department of Electrical and Computer Engineering\\ University of Wisconsin Madison }
\date{}
\maketitle

\begin{abstract} The dueling bandit problem is a variation of the classical multi-armed bandit in which the allowable actions are noisy comparisons between pairs of arms. This paper focuses on a new approach for finding the ``best'' arm according to the Borda criterion using noisy comparisons. We prove that in the absence of structural assumptions, the sample complexity of this problem is proportional to the sum of the inverse squared gaps between the Borda scores of each suboptimal arm and the best arm.  We explore this dependence further and consider structural constraints on the pairwise comparison matrix (a particular form of sparsity natural to this problem) that can significantly reduce the sample complexity.  This motivates a new algorithm called Successive Elimination with Comparison Sparsity (SECS) that exploits sparsity to find the Borda winner using fewer samples than standard algorithms. We also evaluate the new algorithm experimentally with synthetic and real data. The results show that the sparsity model and the new algorithm can provide significant improvements over standard approaches.  \end{abstract}

\section{INTRODUCTION}

The dueling bandit is a variation of the classic multi-armed bandit problem in which the actions are noisy comparisons between arms, rather than observations from the arms themselves \citep{karmedduelingbandit}.  Each action provides $1$ bit indicating which of two arms is probably better.  For example, the arms could represent objects and the bits could be responses from people asked to compare pairs of objects.  
In this paper, we focus on the pure {\em exploration} problem of finding the ``best'' arm from noisy pairwise comparisons. This problem is different from the {\em explore-exploit} problem studied in \citet{karmedduelingbandit}. There can be different notions of ``best'' in the dueling framework, including the Condorcet and Borda criteria (defined below).

Most of the dueling-bandit algorithms are primarily concerned with finding the Condorcet winner
(the arm that is probably as good or better than every other arm).
There are two drawbacks to this.  First, a Condorcet winner does not exist unless the underlying probability matrix governing the outcomes of pairwise comparisons satisfies certain restrictions. These restrictions may not be met in many situations. In fact, we show that a Condorcet winner doesn't exist in our experiment with real data presented below.  Second, the best known upper bounds on the sample complexity of finding the Condorcet winner (assuming it exists) grow quadratically (at least) with the number of arms.  This makes Condorcet algorithms impractical for large numbers of arms.

To address these drawbacks, we consider the Borda criterion instead.  The Borda
score of an arm is the probability that the arm is preferred to another arm
chosen uniformly at random. A Borda winner (arm with the largest Borda score)
always exists for every possible probability matrix. We assume throughout this paper that there exists a unique Borda winner. Finding the Borda winner with probability at least $1-\delta$ can be reduced to solving an instance of the standard multi-armed bandit problem resulting in a sufficient sample complexity of $\mathcal{O}\left(\sum_{i>1} (s_1-s_i)^{-2} \log\left( \log((s_1-s_i)^{-2})/\delta\right) \right)$, where $s_i$ denotes Borda score of arm $i$ and
$s_1>s_2>\cdots>s_n$ are the scores in descending order \citep{karnin2013almost,lilucb}. In favorable cases, for instance, if $s_1-s_i
\geq c$, a constant for all $i>1$, then this sample complexity is linear in $n$ as opposed to the quadratic sample complexity necessary to find the Condorcet winner. In this paper we show that this upper bound is essentially tight, thereby apparently ``closing'' the Borda winner identification problem. However, in this paper we consider a specific type of structure that is motivated by its existence in real datasets that complicates this apparently simple story. In particular, we show that the reduction to a standard multi-armed bandit problem can result in very bad performance when compared to an algorithm that exploits this observed structure.   

We explore the sample complexity dependence in more detail and consider
structural constraints on the matrix (a particular form of sparsity natural to
this problem) that can significantly reduce the sample complexity. The sparsity model
captures the commonly observed behavior in elections in which there are a small
set of ``top'' candidates that are competing to be the winner but only differ on
a small number of attributes, while a large set of ``others'' are mostly irrelevant
as far as predicting the winner is concerned in the sense that they would always
lose in a pairwise matchup against one of the ``top'' candidates.

This motivates a new algorithm called Successive Elimination with Comparison
Sparsity (SECS).  SECS takes advantage of this structure by determining which of
two arms is better on the basis of their performance with respect to a sparse
set of ``comparison'' arms. Experimental results with real data demonstrate the
practicality of the sparsity model and show that SECS can provide significant
improvements over standard approaches. 

The main contributions of this paper are as follows:\\[-14pt]
\begin{itemize}
    \item A distribution dependent lower bound for the sample complexity of identifying the Borda winner that essentially shows that the Borda reduction to the standard multi-armed bandit problem (explained in detail later) is essentially optimal up to logarithmic factors, given no prior structural information. 
    \item A new structural assumption for the $n$-armed
        dueling bandits problem in which the top arms can be distinguished by duels with a sparse set of other arms.
    \item An algorithm for the dueling bandits problem under this assumption, with
        theoretical performance guarantees showing significant sample complexity improvements compared to naive reductions to standard multi-armed bandit algorithms.
   \item Experimental results, based on real-world applications, demonstrating
        the superior performance of our algorithm compared to existing methods.
\end{itemize}

\section{PROBLEM SETUP}
\label{problemsetupsection}
The \textit{n-armed dueling bandits problem} \citep{karmedduelingbandit} is a modification of the \textit{n-armed bandit problem}, where instead of pulling a single arm, we choose a pair of arms $(i,j)$ to duel, and receive one bit indicating which of the two is better or preferred, with the probability of $i$ winning the duel is equal to a constant $p_{i,j}$ and that of $j$ equal to $p_{j,i}=1-p_{i,j}$. We define the \textit{probabilty matrix} $P=[p_{i,j}]$, whose $(i,j)$th entry is $p_{i,j}$.

Almost all existing $n$-armed dueling bandit methods \citep{karmedduelingbandit, beatthemean, rucb,urvoy2013generic,ailon2014reducing} focus on the explore-exploit problem and furthermore make a variety of assumptions on the preference matrix $P$. In
particular, those works assume the existence of a Condorcet winner: an arm, $c$, such that
$p_{c,j}>\frac{1}{2}$ for all $j \neq c$. The \textit{Borda} winner is an arm $b$
that satisfies $\sum_{j\neq b} p_{b,j} \geq \sum_{j\neq i} p_{i,j}$ for all $i=1,\cdots,n$. In
other words, the Borda winner is the arm with the highest average probability of
winning against other arms, or said another way, the arm that has the highest
probability of winning against an arm selected uniformly at random from the
remaining arms. The Condorcet winner has been given more attention than the
Borda, the reasons being: 1) Given a choice between the Borda and the Condorcet
winner, the latter is preferred in a direct comparison between the two. 2) As
pointed out in \citet{urvoy2013generic,rucb} the Borda winner can be found by reducing the dueling bandit problem to a standard multi-armed bandit problem as follows.

\begin{definition} \label{bordareduction}{\em Borda Reduction.} The action of pulling arm $i$ with reward $\frac{1}{n-1}\sum_{j\neq i} p_{i,j}$ can be simulated by dueling arm $i$ with another arm chosen uniformly at random.
\end{definition}
However, we feel that the Borda problem has received far less attention than it
deserves. Firstly, the Borda winner \textit{always exists}, the Condorcet does
not. For example, a Condorcet winner does not exist in the MSLR-WEB10k datasets considered in this paper. Assuming the existence of a Condorcet winner severely restricts the class
of allowed $P$ matrices: only those $P$ matrices are allowed which have a row
with all entries $\geq \frac{1}{2}$. In fact,
\citet{karmedduelingbandit,beatthemean} require that the comparison probabilities
$p_{i,j}$ satisfy additional transitivity conditions that are often violated in
practice. Secondly, there are many cases where the Borda winner and the
Condorcet winner are distinct, and the Borda winner would be preferred in many
cases. Lets assume that arm $c$ is the Condorcet winner, with
$p_{c,i}=0.51$ for $i\neq c$. Let arm $b$ be the Borda winner with $p_{b,i}=1$ for
$i\neq b,c$, and $p_{b,c}=0.49$. It is reasonable that arm $c$ is only marginally
better than the other arms, while arm $b$ is significantly preferred over all
other arms except against arm $c$ where it is marginally rejected. In this
example - chosen extreme to highlight the pervasiveness of situations where the
Borda arm is preferred - it is clear that arm $b$ should be the winner: think of
the arms representing objects being contested such as t-shirt designs, and the
$P$ matrix is generated by showing users a pair of items and asking them to
choose the better among the two. This example also shows that the Borda winner
is more robust to estimation errors in the $P$ matrix (for instance, when the
$P$ matrix is estimated by asking a small sample of the entire population to vote
among pairwise choices). The Condorcet winner is sensitive to entries in
the Condorcet arm's row that are close to $\frac{1}{2}$, which is not the case for
the Borda winner. Finally, there are important cases (explained
next) where the winner can be found in fewer number of duels than would be
required by Borda reduction.

\section{MOTIVATION} \label{motivationSection}

\begin{table*}
\begin{equation}
    P_1 = 
    \begin{blockarray}{cccccccc}
        & 1 & 2 & 3 & \cdots & n & s_i & s_1-s_i\\ 
        \begin{block}{c(ccccc)cc} 1 & \frac{1}{2} & \frac{1}{2}&
            \frac{3}{4} & \cdots & \frac{3}{4}+\epsilon &
            \frac{\frac{1}{2}+\epsilon}{n-1}+\frac{3}{4}\frac{n-2}{n-1} & 0 \\ \\ 2
            & \frac{1}{2}& \frac{1}{2} & \frac{3}{4} & \cdots &
            \frac{3}{4} &
            \frac{\frac{1}{2}}{n-1}+\frac{3}{4}\frac{n-2}{n-1} & \frac{\epsilon}{n-1} \\ \\ 3 &
            \frac{1}{4} & \frac{1}{4} & \frac{1}{2} & \cdots & \frac{1}{2}  &
            \frac{1}{2}\frac{n-2}{n-1} &  \frac{\frac{1}{2}+\epsilon}{n-1}+\frac{1}{4}\frac{n-2}{n-1}  \\ \\ \vdots & \vdots & \vdots & \vdots &
            \ddots & \vdots & \vdots & \vdots \\ \\ n & \frac{1}{4}-\epsilon & \frac{1}{4} &
            \frac{1}{2} & \cdots & \frac{1}{2} &
            -\frac{\epsilon}{n-1}+\frac{1}{2}\frac{n-2}{n-1} &
            \frac{\frac{1}{2}+2\epsilon}{n-1}+\frac{1}{4}\frac{n-2}{n-1} \\
        \end{block}
    \end{blockarray}
    \label{badP1matrix}
\end{equation}
\begin{equation}
    P_2 = 
    \begin{blockarray}{cccccccc}
        & 1 & 2 & 3 & \cdots & n & s_i & s_1-s_i\\ 
        \begin{block}{c(ccccc)cc} 1 & \frac{1}{2} &
            \frac{1}{2}+\frac{\epsilon}{n-1} & \frac{3}{4}+\frac{\epsilon}{n-1}
            & \cdots & \frac{3}{4}+\frac{\epsilon}{n-1} &
            \frac{\frac{1}{2}+\epsilon}{n-1}+\frac{3}{4}\frac{n-2}{n-1} & 0 \\
            \\ 
            2 & \frac{1}{2}-\frac{\epsilon}{n-1} & \frac{1}{2} & \frac{3}{4} &
            \cdots & \frac{3}{4} &
            \frac{\frac{1}{2}-\frac{\epsilon}{n-1}}{n-1}+\frac{3}{4}\frac{n-2}{n-1} &
            \frac{\epsilon}{n-1} + \frac{\epsilon}{(n-1)^2} \\ \\
            3 & \frac{1}{4}-\frac{\epsilon}{n-1} & \frac{1}{4} & \frac{1}{2} &
            \cdots & \frac{1}{2} & \frac{-\frac{\epsilon}{n-1}}{n-1}+
            \frac{1}{2}\frac{n-2}{n-1} &
            \frac{\frac{1}{2}+\epsilon+\frac{\epsilon}{n-1}}{n-1}+
            \frac{1}{4}\frac{n-2}{n-1}  \\ \\ 
            \vdots & \vdots & \vdots & \vdots & \ddots & \vdots & \vdots &
            \vdots \\ \\ n & \frac{1}{4} -\frac{\epsilon}{n-1} & \frac{1}{4} &
            \frac{1}{2} & \cdots & \frac{1}{2} &
            \frac{-\frac{\epsilon}{n-1}}{n-1} + \frac{1}{2}\frac{n-2}{n-1} &
            \frac{\frac{1}{2}+\epsilon+\frac{\epsilon}{n-1}}{n-1}+
            \frac{1}{4}\frac{n-2}{n-1}  \\
        \end{block}
    \end{blockarray}
    \label{badP2matrix}
\end{equation}
\end{table*}

We define the \textit{Borda score} of an arm $i$ to be the probability of the
$i^{\text{th}}$ arm winning a duel with another arm chosen uniformly at random:
\begin{equation*} s_i = \tfrac{1}{n-1}\sum\limits_{j\neq i}p_{i,j} \, .
\end{equation*} Without loss of generality, we assume that $s_1 > s_2 \geq \dots
\geq s_n$ but that this ordering is unknown to the algorithm. As mentioned
above, if the Borda reduction is used then the dueling bandit problem becomes a
regular multi-armed bandit problem and lower bounds for the multi-armed bandit problem
\citep{kaufmann2014complexity,mannortsitsiklis} suggest that the number of samples
required should scale like $\Omega \left( \sum_{i\neq 1}\frac{1}{(s_1-s_i)^2}
\log \frac{1}{\delta} \right)$, which depends only on the Borda scores, and not
the individual entries of the preference matrix. This would imply that  any preference
matrix $P$ with Borda scores $s_i$ is just as hard as another matrix $P'$ with
Borda scores $s_i'$ as long as $(s_1-s_i) = (s_1'-s_i')$. Of course, this lower bound only applies to algorithms using the Borda reduction, and not any algorithm for identifying the Borda winner that may, for instance, collect the duels in a more deliberate way. Next we consider  specific $P$ matrices that exhibit two very different kinds of structure but have the same differences in Borda scores which motivates the structure considered in this paper.  

\subsection{Preference Matrix $P$ known up to permutation of indices}
Shown below in equations (\ref{badP1matrix}) and (\ref{badP2matrix}) are two preference matrices $P_1$ and $P_2$ indexed by the number of arms $n$ that essentially have
the same Borda gaps -- $(s_1-s_i)$ is either like $\frac{\epsilon}{n}$ or
approximately $1/4$ -- but we will argue that $P_1$ is much ``easier'' than $P_2$ in a
certain sense (assume $\epsilon$ is an unknown constant, like $\epsilon =1/5$). Specifically,
if given $P_1$ and $P_2$ up to a permutation of the labels of their indices (i.e. given $\Lambda P_1 \Lambda^T$ for some unknown permutation matrix $\Lambda$), how many comparisons does it take to find the Borda winner in each case for different values of $n$?

Recall from above that if we ignore the fact that we know the matrices up to a
permutation and use the Borda reduction technique, we can use a 
multi-armed bandit algorithm (e.g. \citet{karnin2013almost,lilucb}) and find the best arm for both
$P_1$ and $P_2$ using  $O \left( {n^2} \log(\log(n)) \right)$ samples. We next
argue that given $P_1$ and $P_2$ up to a permutation, there exists an algorithm
that can identify the Borda winner of $P_1$ with just $O(n \log(n))$ samples
while the identification of the Borda winner for $P_2$ requires at least
$\Omega(n^2)$ samples. This shows that given the probability matrices up to a permutation, the sample complexity of identifying the Borda winner does not rely just on the Borda differences, but on the particular structure of the probability matrix. 

Consider $P_1$. We claim that there exists a procedure that exploits the
structure of the matrix to find the best arm of $P_1$ using just $O(n
\log(n))$ samples. Here's how: For each arm, duel it with $32\log
\frac{n}{\delta}$ other arms chosen uniformly at random. By Hoeffding's
inequality, with probability at least $1-\delta$ our empirical estimate of the
Borda score will be within $1/8$ of its true value for all $n$ arms and we can
remove the bottom $(n-2)$ arms due to the fact that their Borda gaps exceed $1/4$.
Having reduced the possible winners to just two arms, we can identify which rows
in the matrix they correspond to and duel each of these two arms against all of the
remaining $(n-2)$ arms $O(\frac{1}{\epsilon^2})$ times to find out which one
has the larger Borda score using just $O\left(
\frac{2(n-2)}{\epsilon^2} \right)$ samples, giving an overall sample complexity
of $O\left( n \log n \right)$. We have improved the sample complexity from $O(n^2 \log(\log(n)))$ using the Borda
reduction to just $O(n \log(n))$.  

Consider $P_2$. We claim that given this matrix up to a permutation of its indices, no algorithm can determine the winner of $P_2$ without
requesting $\Omega(n^2)$ samples.  To see this, suppose an oracle has made the
problem easier by reducing the problem down to just the top two rows of the
$P_2$ matrix. This is a binary hypothesis test for which Fano's inequality implies that to guarantee that the probability of error is not above some constant level, the number of samples to identify the Borda winner must scale like  $\min_{j \in [n] \setminus \{1,2\}} \frac{1}{KL(p_{1,j},p_{2,j})} \geq \min_{j \in [n] \setminus \{1,2\}} \frac{c}{(p_{1,j}-p_{2,j})^2} = \Omega(  (n/\epsilon)^2 )$ where the inequality holds for some $c$ by Lemma~\ref{klbernoulliupperbound} in the Appendix.

We just argued that the
structure of the $P$ matrix, and not just the Borda gaps, can dramatically
influence the sample complexity of finding the Borda winner. This leads us to
ask the question: if we don't know anything about the $P$ matrix beforehand (i.e. do not know the matrix up to a permutation of its indices), can
we learn and exploit this kind of structural information in an online fashion
and improve over the Borda reduction scheme? The answer is no, as we
argue next.

\subsection{Distribution-Dependent Lower Bound}
\label{lowerboundsection}
We prove a distribution-dependent lower bound on the complexity of finding the
best Borda arm for a general $P$ matrix. This is a result important in its own
right as it shows that the lower bound obtained for an algorithm using the Borda reduction is tight, that is, this result implies that barring any structural assumptions, the Borda reduction is optimal. 
\begin{definition}
{\em $\delta$-PAC dueling bandits algorithm}: A $\delta$-PAC dueling
bandits algorithm is an algorithm that selects duels between arms and based on the outcomes finds the Borda winner with probability
greater than or equal to $1-\delta$. 
\end{definition}
The techniques used to prove the following result are inspired from Lemma 1 in
\citet{kaufmann2014complexity} and Theorem 1 in \citet{mannortsitsiklis}. 

\begin{theorem}
    (Distribution-Dependent Lower Bound)
    Consider a matrix $P$ such that $\frac{3}{8} \leq p_{i,j} \leq \frac{5}{8},
    \forall i,j \in [n]$ with $n \geq 4$. Let $\tau$ be the total number of duels. Then for
    $\delta \leq 0.15$, any $\delta$-PAC dueling bandits algorithm to find the
    Borda winner has $$\E_P[\tau] \geq C \log
    \frac{1}{2\delta}
    \sum\limits_{i \neq 1}^{} \frac{1}{(s_1-s_i)^2} $$ where $s_i = \frac{1}{n-1}\sum_{j \neq
    i} p_{i,j}$ denotes the Borda score of arm $i$. Furthermore, $C$ can be chosen to be
    $1/90$.\\
\end{theorem}

The proof can be found in the supplementary material.

In particular, this implies that for the preference matrix $P_1$ in
\eqref{badP1matrix}, any algorithm that makes no assumption about the structure
of the $P$ matrix requires $\Omega \left( n^2 \right)$ samples.  Next we argue that the particular structure found in $P_1$ is an extreme case of a more general structural phenomenon found in real datasets and that it is a natural structure to assume and design algorithms to exploit.

\subsection{Motivation from Real-World Data}
The matrices $P_1$ and $P_2$ above illustrate a key structural aspect that can
make it easier to find the Borda winner.  If the arms with the top Borda scores
are distinguished by duels with a small subset of the arms (as exemplified in
$P_1$), then finding the Borda winner may be easier than in the general case.
Before formalizing a model for this sort of structure, let us look at two
real-world datasets, which motivate the model.  

We consider the Microsoft Learning to Rank web search datasets MSLR-WEB10k
\citep{letor} and MQ2008-list \citep{letor2} (see the experimental section for a
descrptions).  Each dataset is used to construct a corresponding probability
matrix $P$.  We use these datasets to test the hypothesis that comparisons with
a small subset of the arms may suffice to determine which of two arms has a
greater Borda score.  

Specifically, we will consider the Borda score of the best arm (arm $1$) and
every other arm.
 For any other arm $i>1$ and any positive integer $k
\in [n-2]$,
let $\Omega_{i,k}$ be a set of cardinality $k$ containing the indices $j \in [n]\setminus \{1,i\}$ with the $k$ largest
discrepancies $|p_{1,j}-p_{i,j}|$. These are the duels that, individually,
display the greatest differences between arm $1$ and $i$.  For each $k$, define
$\alpha_i(k)=2(p_{1,i}-\tfrac{1}{2}) + \sum_{j\in \Omega_{i,k}} (p_{1,j}-p_{i,j})$.  If the hypothesis
holds, then the duels with a small number of (appropriately chosen) arms should
indicate that arm $1$ is better than arm $i$.  In other words, $\alpha_i(k)$
should become and stay positive as soon as $k$ reaches a relatively small
value.
 Plots of these $\alpha_i$ curves for two datasets are presented in
Figures~\ref{modelFit}, and indicate that the Borda winner is apparent for small
$k$.  This behavior is explained by the fact that the individual discrepancies
$|p_{1,j}-p_{i,j}|$, decay quickly when ordered from largest to smallest, as
shown in Figure~\ref{des}.  

The take away message is that it is unnecessary to estimate the difference  or
gap between the Borda scores of two arms. It suffices to compute the {\em
partial} Borda gap based on duels with a small subset of the arms.  An
appropriately chosen subset of the duels will correctly indicate which arm has a
larger Borda score. The algorithm proposed in the next section automatically
exploits this structure.

\begin{figure}[h]
\centering
\includegraphics[width=8cm]{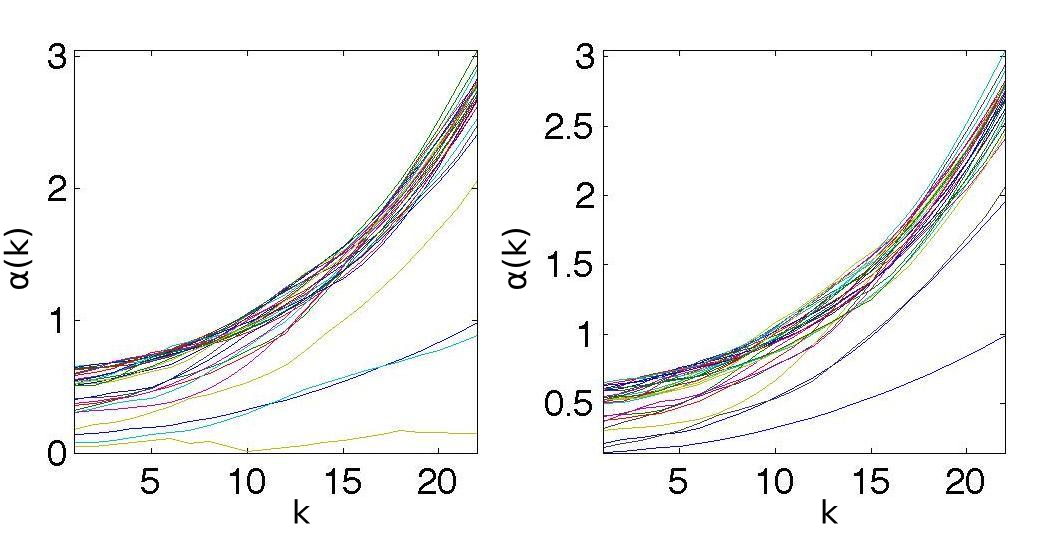} 
\caption{Plots of $\alpha_i(k)=2(p_{1,i}-\tfrac{1}{2}) + \sum_{j\in \Omega_{i,k}} (p_{1,j}-p_{1,j})$ vs.
$k$ for $30$ randomly chosen arms (for visualization purposes); MSLR-WEB10k on
left, MQ2008-list on right. The curves are strictly positive after a small
number of duels.}
\label{modelFit}
\end{figure}

\begin{figure}[h]
\centering
\includegraphics[width=8cm]{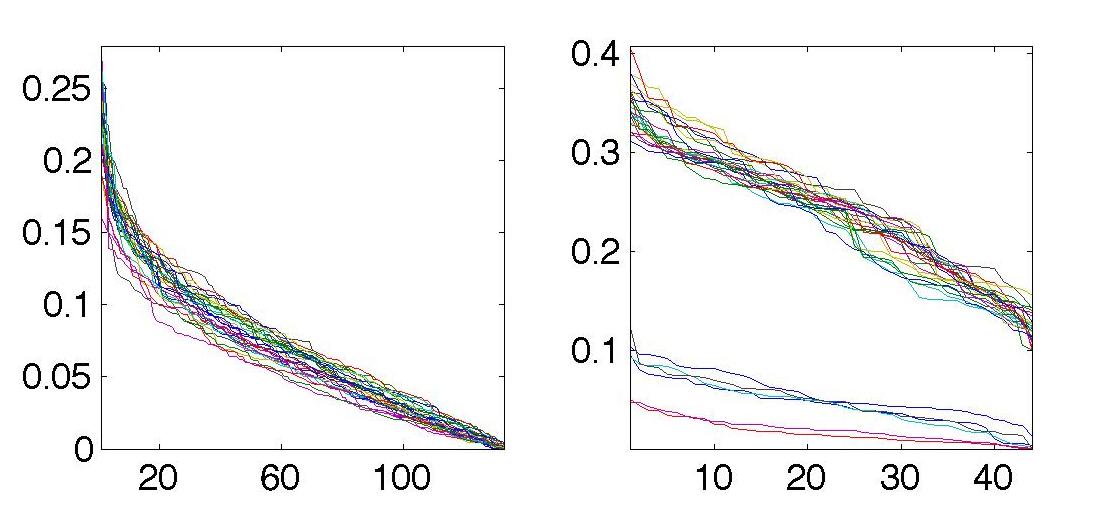} 
\caption{Plots of discrepancies $|p_{1,j}-p_{i,j}|$ in descending order for $30$
randomly chosen arms (for visualization purposes); MSLR-WEB10k on left,
MQ2008-list on right.}
\label{des}
\end{figure}

\section{ALGORITHM AND ANALYSIS}

\begin{algorithm*}
Input sparsity level $k \in [n-2]$, time gate $T_0 \geq 0$ \\
Start with active set $A_1 = \{1,2,\cdots,n \}$, $t=1$\\
Let $C_t = \sqrt{ \frac{2\log(4 n^2 t^2 / \delta)}{t/n} } + \frac{ 2 \log(4n^2t^2/\delta)}{3t/n}$\\
\While{$|A_t|>1$}{
Choose $I_t$ uniformly at random $[n]$.\\
\For{$j \in A_t$}{
%Observe $Z_{j,I_t}^{(t)}$ and update $\widehat{s}_{j,t} = \frac{1}{t} \sum_{\ell=1}^t Z_{j,I_\ell}^{(\ell)}$\\
Observe $Z_{j,I_t}^{(t)}$ and update $\widehat{p}_{j,I_t,t} = \frac{n}{t}
\sum_{\ell=1}^t  Z_{j,I_\ell}^{(\ell)} \1_{I_\ell = I_t}$, $\widehat{s}_{j,t} =
\frac{   n/(n-1)  }{t} \sum_{\ell=1}^t Z_{j,I_\ell}^{(\ell)}$.\\
}
$\displaystyle A_{t+1} = A_t \setminus \bigg\{ j \in A_t : \exists i \in A_t \text{ with }\ \ $\\ \quad\, \hspace{.4in} $ \, 1) \ \ \ \1_{\{t > T_0\}} \, \widehat{\Delta}_{i,j,t}\left(  \arg\max_{\Omega \subset [n] : |\Omega| = k} \widehat{\nabla}_{i,j,t}(\Omega)  \right) ~> ~6 (k+1) C_t$ \\ \indent\hspace{.3in}$\text{  \textbf{OR}  } \ \ \ 2) \ \  \ \widehat{s}_{i,t} > \widehat{s}_{j,t} + \frac{n}{n-1}\sqrt{ \frac{2 \log(4 n t^2/\delta)}{t}  } \bigg\}$\\
$t \leftarrow t+1$\\
}
\caption{Sparse Borda Algorithm}
\end{algorithm*}

In this section we propose a new algorithm that exploits the kind of structure just described above and prove a sample complexity bound. The algorithm is inspired by the Successive Elimination (SE) algorithm of \citet{even2006action} for standard multi-armed bandit problems. Essentially, the proposed algorithm below implements SE with the Borda reduction and an additional elimination criterion that exploits sparsity (condition 1 in the algorithm). We call the algorithm Successive Elimination with Comparison Sparsity (SECS). 

We will use $\1_E$ to denote the indicator of the event $E$ and $[n] = \{1,2,\dots,n\}$. The algorithm maintains an active set of arms $A_t$ such that if $j \notin A_t$ then the algorithm has concluded that arm $j$ is not the Borda winner. At each time $t$, the algorithm chooses an arm $I_t$ uniformly at random from $[n]$ and compares it with all the arms in $A_t$. Note that $A_{k} \subseteq A_{\ell}$ for all $k \geq \ell$. Let $Z_{i,j}^{(t)} \in \{0,1\}$ be independent Bernoulli random variables with $\E[Z_{i,j}^{(t)}] = p_{i,j}$, each denoting the outcome of ``dueling'' $i,j \in [n]$ at time $t$ (define $Z_{i,j}^{(t)}=0$ for $i=j$). For any $t \geq 1$, $i \in [n]$, and $j \in A_t$ define
\begin{align*}
\widehat{p}_{j,i,t} = \frac{n}{t}
\sum_{\ell=1}^t  Z_{j,I_\ell}^{(\ell)} \1_{I_\ell = i}
\end{align*}
so that $\E\left[ \widehat{p}_{j,i,t} \right] = p_{j,i}$. Furthermore, for any $t \geq1$, $j \in A_t$ define 
\begin{align*}
\widehat{s}_{j,t} =
\frac{ n/(n-1) }{t} \sum_{\ell=1}^t Z_{j,I_\ell}^{(\ell)}
\end{align*}
so that $\E\left[ \widehat{s}_{j,t} \right] = s_{j}$. 
For any $\Omega \subset [n]$ and $i,j \in [n]$ define  
\begin{align*}
\Delta_{i,j}(\Omega) &= 2(p_{i,j}-\tfrac{1}{2}) + \sum_{\omega \in \Omega: \omega \neq i \neq j} ( p_{i,\omega} - p_{j,\omega} )\\
\widehat{\Delta}_{i,j,t}(\Omega) &=  2( \widehat{p}_{i,j,t} -\tfrac{1}{2})  + \sum_{\omega \in \Omega: \omega \neq i \neq j} ( \widehat{p}_{i,\omega,t} - \widehat{p}_{j,\omega,t} ) \\
{\nabla}_{i,j}(\Omega) &=   \sum_{\omega \in \Omega: \omega \neq i \neq j} | {p}_{i,\omega} - {p}_{j,\omega} | \\
\widehat{\nabla}_{i,j}(\Omega) &=  \sum_{\omega \in \Omega: \omega \neq i \neq j} | \widehat{p}_{i,\omega,t} - \widehat{p}_{j,\omega,t} | \ .
\end{align*}
The quantity $\Delta_{i,j}(\Omega)$ is the {\em partial} gap between the Borda scores for $i$ and $j$, based on only the comparisons with the arms in $\Omega$.  Note that $\frac{1}{n-1}\Delta_{i,j}([n]) = s_i - s_j$. 
The quantity $ \arg\max_{\Omega \subset [n] : |\Omega| = k} {\nabla}_{i,j}(\Omega) $ selects the indices $\omega$ yielding the largest discrepancies $|p_{i,\omega} -p_{j,\omega}|$. $\widehat{\Delta}$ and $\widehat{\nabla}$ are empirical analogs of these quantities.
  
\begin{definition}
For any $i \in [n] \setminus 1$ we say the set $\{ (p_{1,\omega}-p_{i,\omega}) \}_{\omega \neq 1 \neq i}$ is $(\gamma,k)$-approximately sparse if 
\begin{align*}
\displaystyle \max_{\Omega \in [n] : |\Omega| \leq k}\nabla_{1,i}(\Omega \setminus \Omega_i)  \ \leq \ \gamma \Delta_{1,i}(\Omega_i)
\end{align*} 
where $\displaystyle\Omega_i =  \arg\max_{\Omega \subset[n] : |\Omega|=k} \nabla_{1,i}( \Omega )$.
\end{definition}
Instead of the strong assumption that the set $\{ (p_{1,\omega}-p_{i,\omega}) \}_{\omega \neq 1 \neq i}$ has no more than $k$ non-zero coefficients, the above definition relaxes this idea and just assumes that the absolute value of the coefficients outside the largest $k$ are small relative to the partial Borda gap. This definition is inspired by the structure described in previous sections and will allow us to find the Borda winner faster. 

The parameter $T_0$ is specified (see Theorem~\ref{sparseTheorem}) to guarantee that all arms with sufficiently large gaps $s_1-s_i$ are eliminated by time step $T_0$ (condition 2). Once $t > T_0$, condition 1 also becomes active and the algorithm starts removing arms with large partial Borda gaps, exploiting the assumption that the top arms can be distinguished by comparisons with a sparse set of other arms. The algorithm terminates when only one arm remains.

\begin{theorem} \label{sparseTheorem}
Let $k \geq 0$ and $T_0>0$ be inputs to the above algorithm and let $R$ be the solution to $\frac{32 }{R^2} \log\left( \frac{32 n/ \delta}{R^2 }  \right) = T_0$. If for all $i \in [n] \setminus 1$, at least one of the following holds:
\begin{enumerate}
\item$\{ (p_{1,\omega}-p_{i,\omega}) \}_{\omega \neq 1 \neq i}$ is $(\tfrac{1}{3},k)$-approximately sparse, 
\item $(s_1-s_i) \geq R$, 
\end{enumerate} 
then with probability at least $1- 3\delta$, the algorithm returns the best arm after no more than
\begin{align*}
c \sum_{j > 1} \min\Big\{  \max\left\{  \tfrac{1 }{R^2} \log\left( \tfrac{ n/
\delta}{R^2 }  \right), \tfrac{ (k+1)^2 / n }{ \Delta_{j}^2} \log\left(  \tfrac{
n / \delta }{ \Delta_{j}^{2} } \right)  \right\} ,\\
\tfrac{ 1 }{ \Delta_{j}^{2}}
\log\left(  \tfrac{ n / \delta }{ \Delta_{j}^{2} } \right)  \Big\}
\end{align*}
samples where $\Delta_{j}:= s_{1} - s_j$ and $c>0$ is an absolute constant.
\end{theorem}

The second argument of the $\min$ is precisely the result one would obtain by running Successive Elimination with the Borda reduction \citep{even2006action}. Thus, under the stated assumptions, the algorithm never does worse than the Borda reduction scheme.  The first argument of the $\min$ indicates the potential improvement gained by exploiting the sparsity assumption.  The first argument of the $\max$ is the result of throwing out the arms with large Borda differences and the second argument is the result of throwing out arms where a {\em partial} Borda difference was observed to be large. 

To illustrate the potential improvements, consider the $P_1$ matrix discussed above, the theorem implies that by setting $T_0 = \frac{32 }{R^2} \log\left( \frac{32 n/ \delta}{R^2 }  \right) $ with $R = \frac{1/2 + \epsilon}{n-1} + \frac{1}{4} \frac{n-2}{n-1} \approx \frac{1}{4}$ and $k=1$ we obtain a sample complexity of $O( \epsilon^{-2} n \log(n))$ for the proposed algorithm compared to the standard Borda reduction sample complexity of $\Omega(n^2)$. 

In practice it is difficult optimize the choice of  $T_0$ and $k$, but motivated by the results shown in the experiments section, we recommend setting $T_0=0$ and $k=5$ for typical problems.

\section{EXPERIMENTS}
The goal of this section is not to obtain the best possible sample complexity results for the specified datasets, but to show the {\em relative} performance gain of exploiting structure using the proposed SECS algorithm with respect to the Borda reduction. That is, we just want to measure the effect of exploiting sparsity while keeping all other parts of the algorithms constant. Thus, the algorithm we compare to that uses the simple Borda reduction is simply the SECS algorithm described above but with $T_0 = \infty$ so that the sparse condition never becomes activated. Running the algorithm in this way, it is very closely related to the Successive Elimination algorithm of \citet{even2006action}. In what follows, our proposed algorithm will be called SECS and the benchmark algorithm will be denoted as just the Borda reduction (BR) algorithm. 

We experiment on both simulated data and two real-world datasets. During all experiments, both the BR and SECS algorithms were run with $\delta=0.1$. For the SECS algorithm we set $T_0=0$ to enable condition 1 from the very beginning (recall for BR we set $T_0 = \infty$). Also, while the algorithm has a constant factor of 6 multiplying $(k+1)C_t$, we feel that the analysis that led to this constant is very loose so in practice we recommend the use of a constant of $1/2$ which was used in our experiments. While the change of this constant invalidates the guarantee of Theorem~\ref{sparseTheorem}, we note that in all of the experiments to be presented here, neither algorithm ever failed to return the best arm. This observation also suggests that the SECS algorithm is robust to possible inconsistencies of the model assumptions.
 
\subsection{Synthetic Preference matrix}
\label{synth}
Both algorithms were tasked with finding the best arm using the $P_1$ matrix of \eqref{badP1matrix} with $\epsilon = 1/5$ for problem sizes equal to  $n=10,20,30,40,50,60,70,80$ arms. Inspecting the $P_1$ matrix, we see that a value of $k=1$ in the SECS algorithm suffices so this is used for all problem sizes. The entries of the preference matrix $P_{i,j}$ are used to simulate comparisons between the respective arms and each experiment was repeated 75 times.

\begin{figure}
\centering
\includegraphics[trim=0cm 0cm 0cm 0cm, clip=true, scale = 0.4]{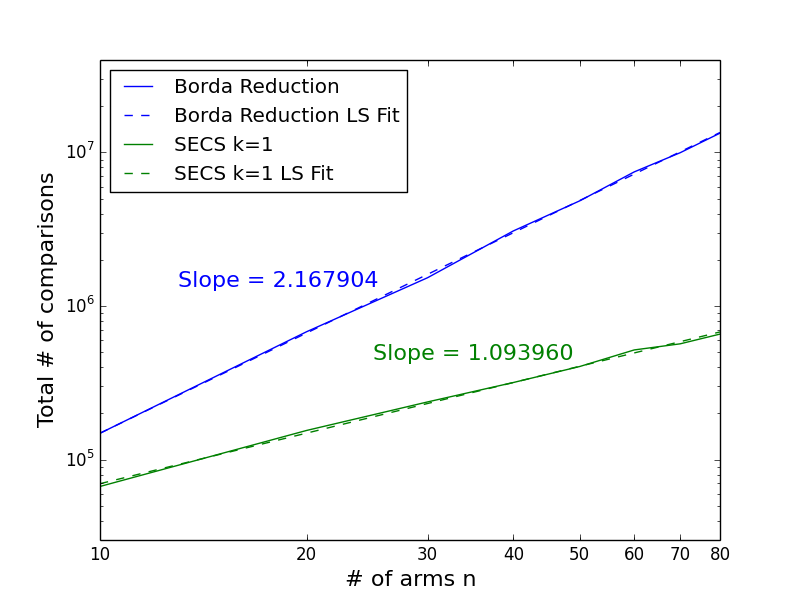}
\caption{Comparison of the Borda reduction algorithm and the proposed SECS algorithm ran on the $P_1$ matrix for different values of $n$. Plot is on log-log scale so that the sample complexity grows like $n^s$ where $s$ is the slope of the line.}
%\end{center}
\label{SyntheticResults}
\end{figure}

Recall from Section~\ref{motivationSection} that any algorithm using the Borda reduction on the $P_1$ matrix has a sample complexity of $\Omega(n^2)$. Moreover, inspecting the proof of Theorem~\ref{sparseTheorem} one concludes that the BR algorithm has a sample complexity of $O(n^2 \log(n))$ for the $P_1$ matrix. On the other hand, Theorem~\ref{sparseTheorem} states that the SECS algorithm should have a sample complexity no worse than $O( n \log(n) )$ for the $P_1$ matrix. Figure \ref{SyntheticResults} plots the sample complexities of SECS and BR on a log-log plot. On this scale, to match our sample complexity hypotheses, the slope of the BR line should be about $2$ while the slope of the SECS line should be about $1$, which is exactly what we observe. 

\subsection{Web search data}
% Cite LeTor data set, Hoffman
We consider two web search data sets. The first is the MSLR-WEB10k Microsoft Learning to Rank data set \citep{letor} that is characterized by approximately 30,000 search queries over a number of documents from search results. The data also contains the values of 136 features and corresponding user labelled relevance factors with respect to each query-document pair. We use the training set of Fold 1, which comprises of about 2,000 queries. The second data set is the MQ2008-list from the Microsoft Learning to Rank 4.0 (MQ2008) data set \citep{letor2}. We use the training set of Fold 1, which has about 550 queries. Each query has a list of documents with 46 features and corresponding user labelled relevance factors.

For each data set, we create a set of rankers, each corresponding to a feature from the feature list. The aim of this task is be to determine the feature whose ranking of query-document pairs is the most relevant. To compare two rankers, we randomly choose a pair of documents and compare their relevance rankings with those of the features. Whenever a mismatch occurs between the rankings returned by the two features, the feature whose ranking matches that of the relevance factors of the two documents ``wins the duel''. If both features rank the documents similarly, the duel is deemed to have resulted in a tie and we flip a fair coin. We run a Monte Carlo simulation on both data sets to obtain a preference matrix $P$ corresponding to their respective feature sets. As with the previous setup, the entries of the preference matrices ($[P]_{i,j}=p_{i,j}$) are used to simulate comparisons between the respective arms and each experiment was repeated 75 times.

From the MSLR-WEB10k data set, a single arm was removed for our experiments as its Borda score was unreasonably close to the arm with the best Borda score and behaved unlike any other arm in the dataset with respect to its $\alpha_i$ curves, confounding our model. % diff was .0034
For these real datasets, we consider a range of different $k$ values for the SECS algorithm. As noted above, while there is no guarantee that the SECS algorithm will return the true Borda winner, in all of our trials for all values of $k$ reported we never observed a single error. This is remarkable as it shows that the correctness of the algorithm is insensitive to the value of $k$ on at least these two real datasets. The sample complexities of BR and SECS on both datasets are reported in Figure~\ref{realData}. We observe that the SECS algorithm, for small values of $k$, can identify the Borda winner using as few as {\em half} the number required using the Borda reduction method. As $k$ grows, the performance of the SECS algorithm becomes that of the BR algorithm, as predicted by Theorem~\ref{sparseTheorem}. 

\begin{figure}
\centering
%\framebox[4.0in]{$\;$}
\begin{subfloat}[MSLR-WEB10k]{
%\begin{center}
\centering
%\framebox[4.0in]{$\;$}
\includegraphics[trim=0cm 0cm 0cm 0cm, clip=true, scale = 0.4]{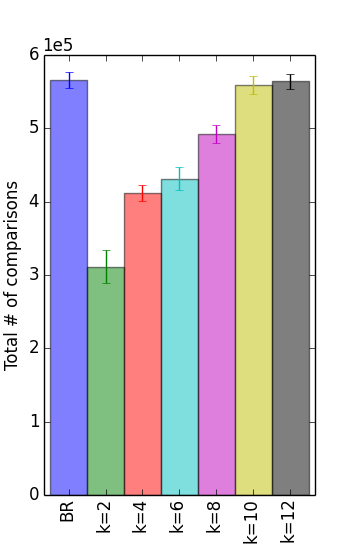}
}
\begin{subfloat}[MQ2008]{
\centering
\includegraphics[trim=.25cm 0cm 0cm 0cm, clip=true, scale = 0.4]{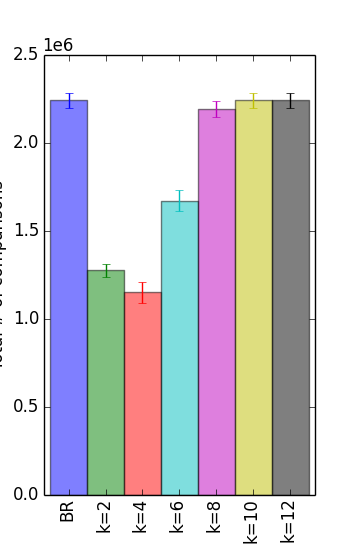}
}
%\fbox{\rule[-.5cm]{0cm}{4cm} \rule[-.5cm]{4cm}{0cm}}
%\end{center}
\end{subfloat}%
%\fbox{\rule[-.5cm]{0cm}{4cm} \rule[-.5cm]{4cm}{0cm}}
%\end{center}
\end{subfloat}
\caption{Comparison of an action elimination-style algorithm using the Borda reduction (denoted as BR) and the proposed SECS algorithm with different values of $k$ on the two datasets.}
%\end{center}
\label{realData}
\end{figure}

Lastly, the preference matrices of the two data
sets support the argument for finding the Borda winner over the Condorcet
winner. The MSLR-WEB10k data set has no Condorcet winner arm. However, while the MQ2008 data set has a Condorcet winner,  when we consider the Borda scores of the arms, it ranks second.

\bibliography{sparse_dueling_bandits}

\appendix

\section{Proof of Lower Bound} \label{lowerBound}
We begin by stating a few technical lemmas. At the heart of the proof of the lower bound  is Lemma 1 of  \citet{kaufmann2014complexity} restated here for completeness.
\begin{lemma}
\label{kaufmannlemma}
Let $\nu$ and $\nu'$ be two bandit models defined over $n$ arms. Let $\sigma$ be a stopping time with respect to
$(\mathcal{F}_t)$ and let $A \in \mathcal{F}_{\sigma}$ be an event such that
$0 < \P_{\nu}(A) < 1$. Then 
$$\sum\limits_{a=1}^n \E_\nu[N_a(\sigma)]KL(\nu_a,\nu_a') \geq
d(\P_{\nu}(A),\P_{\nu'}(A))$$
where $d(x,y) = x \log (x/y) + (1-x)\log ((1-x)/(1-y))$.
\end{lemma} 

Note that the function $d$ is exactly the KL-divergence between two Bernoulli distributions.
\begin{corollary}
\label{kaufmanncorollary}
Let $N_{i,j}=N_{j,i}$ denote the number of duels between arms $i$ and $j$. For the duelling bandits problem with $n$ arms, we have $\frac{(n-1)(n-2)}{2}$
free parameters (or arms). These are the numbers in the upper triangle of the
$P$ matrix. Then, if $P'$ is an alternate matrix, we have from Lemma
\ref{kaufmannlemma}, $$\sum\limits_{i=1}^n \sum\limits_{j=i+1}^n \E_P[N_{i,j}]
d(p_{i,j},p_{i,j}') \geq d(\P_P(A),\P_{P'}(A))$$ 
\end{corollary}
The above corollary relates the cumulative number of duels of a subset of arms to the uncertainty between the actual distribution and an alternative distribution. In deference to interpretability rather than preciseness, we will use the following bound of the KL divergence. 

\begin{lemma} (Upper bound on KL Divergence for Bernoullis)
\label{klbernoulliupperbound}
Consider two Bernoulli random variables with means $p$ and $q$, $0 < p,q < 1$.
Then 
$$d(p,q) \leq \frac{(p-q)^2}{q(1-q)}$$
\end{lemma}
\begin{proof}
\begin{align*}
    d(p,q) = p \log \frac{p}{q} + (1-p)\log \frac{1-p}{1-q} \leq p
    \frac{p-q}{q} + (1-p)\frac{q-p}{1-q} = \frac{(p-q)^2}{q(1-q)}
\end{align*}
where we use the fact that $\log x \leq x-1$ for $x > 0$.
\end{proof}

We are now in a position to restate and prove the lower bound theorem.

\begin{theorem}
(Lower bound on sample complexity of finding Borda winner for the Dueling
Bandits Problem) Consider a matrix $P$ such that $\frac{3}{8} \leq p_{i,j} \leq
\frac{5}{8}, \forall i,j \in [n]$, and $n \geq 3$. Then for $\delta \leq 0.15$, any $\delta$-PAC
dueling bandits algorithm to find the Borda winner has $$\E_P[\tau] \geq C \left(
\frac{n-2}{n-1} \right)^2 \left( \sum\limits_{i \neq 1}^{} \frac{1}{(s_1-s_i)^2}
\right) \log \frac{1}{2\delta}$$ where $s_i = \frac{1}{n-1}\sum\limits_{j \neq
i}^{}p_{i,j}$ denotes the Borda score of arm $i$. $C$ can be chosen to be
$\tfrac{1}{40}$.
\end{theorem}

\begin{proof}
Consider an alternate hypothesis $P'$ where arm $b$ is the best arm, and such
that $P'$ differs from $P$ only in the indices $\{bj: j \notin \{1,b\}\}$.  Note
that the Borda score of arm 1 is unaffected in the alternate hypothesis.
Corollary \ref{kaufmanncorollary} then gives us: 
\begin{align}
\sum\limits_{j \in [n] \backslash \{1,b\}} \E_P[N_{b,j}]d(p_{b,j},p_{b,j}') \geq
d(\P(A),\P(A')) \label{lowerboundeqn1}
\end{align}
Let $A$ be the event that the algorithm selects arm $1$ as the best arm. Since
we assume a $\delta$-PAC algorithm, $\P_P(A) \geq 1-\delta$, $\P_{P'}(A) \leq
\delta$. It can be shown that for $\delta \leq 0.15$, $d(\P_P(A),\P_{P'}(A)) \geq
\log \frac{1}{2\delta}$.

Define $N_b = \sum\limits_{j \neq b}^{} N_{b,j}$. Consider
\begin{align}
    \left( \max\limits_{j \notin \{1,b\}}^{}
    \frac{(p_{b,j}-p'_{b,j})^2}{p'_{b,j}(1-p'_{b,j})} \right) \E_P[N_b] &\geq \left(
    \max\limits_{j \notin \{1,b\}}^{} d(p_{b,j},p'_{b,j}) \right) \E_P[N_b] \nonumber \\
    &= \left( \max\limits_{j \notin \{1,b\}}^{} d(p_{b,j},p'_{b,j}) \right)
    \left( \sum\limits_{j \neq b}^{} \E_P[N_{b,j}] \right) \nonumber \\
    &\geq \left( \max\limits_{j \notin \{1,b\}}^{} d(p_{b,j},p'_{b,j}) \right)
    \left( \sum\limits_{j \notin \{1,b\}}^{} \E_P[N_{b,j}] \right) \nonumber \\
    &\geq \sum\limits_{j \in [n] \backslash \{1,b\}}
    \E_P[N_{b,j}]d(p_{b,j},p_{b,j}') \nonumber \\
    &\geq \log \frac{1}{2\delta} .\quad \text{(by \eqref{lowerboundeqn1})} \label{lowerboundeqn2} 
\end{align}

In particular, choose $p'_{b,j} = p_{b,j}+\frac{n-1}{n-2}(s_1-s_b)+\varepsilon$,
$j \notin \{1,b\}$. As required, under hypothesis $P'$, arm $b$ is the best
arm.\\

Since $p_{b,j} \leq \frac{5}{8}$, $s_1 \leq \frac{5}{8}$, and $s_b \geq
\frac{3}{8}$, as $\varepsilon \searrow 0$,
$\lim\limits_{\varepsilon \searrow 0} p'_{b,j} \leq
\frac{15}{16}$. This implies $\frac{1}{p'_{b,j}(1-p'_{b,j})} \leq
\frac{256}{15} \leq 20$. \eqref{lowerboundeqn2} implies
\begin{align}
    &\phantom{\Rightarrow} 20 \left( \frac{n-1}{n-2}(s_1-s_b) + \varepsilon
    \right)^2 \E_P[N_b] \geq \log \frac{1}{2\delta} \nonumber \\
    &\Rightarrow \E_P[N_b] \geq \frac{1}{20} \left( \frac{n-2}{n-1} \right)^2
    \frac{1}{(s_1-s_b)^2} \log \frac{1}{2\delta} \label{lowerboundeqn3}
\end{align}
where we let $\varepsilon \searrow 0$.\\

Finally, iterating over all arms $b \neq 1$, we have
$$\E_P[\tau]  = \frac{1}{2} \sum\limits_{b=1}^n \sum\limits_{j\neq b}
\E_P[N_{b,j}] = \frac{1}{2} \sum\limits_{b=1}^n \E_P[N_b] \geq
\frac{1}{2}\sum\limits_{b=2}^{n} \E_P[N_b] \geq \frac{1}{40}
\left( \frac{n-2}{n-1} \right)^2 \left( \sum\limits_{b \neq 1}^{}
\frac{1}{(s_1-s_b)^2} \right) \log \frac{1}{2\delta}$$
\end{proof}

\section{Proof of Upper Bound} \label{upperBound}

To prove the theorem we first need a technical lemma.
\begin{lemma} \label{concentrationLemma}
For all $s \in \mathbb{N}$, let $I_s$ be drawn independently and uniformly at random from $[n]$ and let $Z_{i,j}^{(s)}$ be a Bernoulli random variable with mean $p_{i,j}$. If $\widehat{p}_{i,j,t} = \frac{n}{t} \sum_{s=1}^t Z_{i,j}^{(s)} \1_{I_s = j}$ for all $i \in [n]$ and $C_t =  \sqrt{ \frac{2\log(4 n^2 t^2 / \delta)}{t/n} } + \frac{ 2 \log(4n^2t^2/\delta)}{3t/n}$ then $\P\left( \bigcup_{(i,j) \in [n]^2 : i \neq j} \bigcup_{t=1}^\infty \left\{ \left| \widehat{p}_{i,j,t} - p_{i,j} \right| >  C_t   \right\}   \right)  \leq \delta$. 
\end{lemma}
\begin{proof}
Note that $t \widehat{p}_{i,j,t} = \sum_{s=1}^t  n Z_{i,j}^{(s)} \1_{I_s = j}$ is a sum of i.i.d. random variables taking values in $[0,n]$ with $\E\left[ \left(n Z_{i,j}^{(s)} \1_{I_s = j} \right)^2 \right] \leq n^2 \E\left[  \1_{I_s = j}  \right] \leq n$. A direct application of Bernstein's inequality \citep{boucheron2013concentration} and union bounding over all pairs $(i,j) \in [n]^2$ and time $t$ gives the result. 
\end{proof}
A consequence of the lemma is that by repeated application of the triangle inequality,
\begin{align*}
\left| \widehat{\nabla}_{i,j,t}(\Omega) - \nabla_{i,j}(\Omega) \right| &=  \bigg| \sum_{\omega \in \Omega: \omega \neq i \neq j} | \widehat{p}_{i,\omega,t} - \widehat{p}_{j,\omega,t} | - |p_{i,\omega} - p_{j,\omega}| \bigg| \\
&\leq  \sum_{\omega \in \Omega: \omega \neq i \neq j} | \widehat{p}_{i,\omega,t} - p_{i,\omega}| + | p_{j,\omega} - \widehat{p}_{j,\omega,t} |  \\   
&\leq 2 |\Omega|   C_t
\end{align*}
and similarly $\left| \widehat{\Delta}_{i,j,t}(\Omega) - \Delta_{i,j}(\Omega) \right| \leq 2(1+|\Omega|)C_t$  for all $i,j \in [n]$ with $i \neq j$, all $t \in \mathbb{N}$ and all $\Omega \subset [n]$. We are now ready to prove Theorem~\ref{sparseTheorem}.
\begin{proof}
We begin the proof by defining $C_t(\Omega) = 2(1 +  |\Omega| )  C_t$ and considering the events 
\begin{align*}
&\bigcap_{t=1}^\infty \bigcap_{\Omega \subset [n] }  \left\{  | \widehat{\Delta}_{i,j,t}(\Omega) - \Delta_{i,j}(\Omega) |  <   C_{t}(\Omega) \right\} \, ,\\
&\bigcap_{t=1}^\infty \bigcap_{\Omega \subset [n] }  \left\{  | \widehat{\nabla}_{i,j,t}(\Omega) - \nabla_{i,j}(\Omega) |  <   C_{t}(\Omega) \right\} \, ,\\
&\bigcap_{t=1}^\infty \bigcap_{i=1}^n \left\{  | \widehat{s}_{i,t} - s_i | <  \frac{n}{n-1 } \sqrt{ \frac{ \log(4 n t^2/\delta)}{2t}  } \right\} \, ,
\end{align*}
that each hold with probability at least $1-\delta$. The first set of events are a consequence of Lemma~\ref{concentrationLemma} and the last set of events are proved using a straightforward Hoeffding bound \citep{boucheron2013concentration} and a union bound similar to that in Lemma~\ref{concentrationLemma}. In what follows assume these events hold. 

\textbf{Step 1: If $t > T_0$ and $s_1-s_j > R$, then $j \notin A_t$.}\\
We begin by considering all those $j \in [n] \setminus 1$ such that  $s_1-s_j \geq R$ and show that with the prescribed value of $T_0$, these arms are thrown out before $t > T_0$. By the events defined above, for arbitrary $i \in [n] \setminus 1$ we have
\begin{align*}
\widehat{s}_{i,t} - \widehat{s}_{1,t} = \widehat{s}_{i,t} - s_i + s_1 - \widehat{s}_{1,t}  + s_i-s_1 \leq s_i - s_{1} + \frac{2n}{n-1 } \sqrt{ \frac{ \log(4 n t^2/\delta)}{2t}  } \leq \frac{2n}{n-1} \sqrt{ \frac{ \log(4 n t^2/\delta)}{2t}  } 
\end{align*}
since by definition $s_1 > s_{i}$. This proves that the best arm will never be thrown out using the Borda reduction which implies that $1 \in A_t$ for all $t \leq T_0$. On the other hand, for any $j \in [n] \setminus 1$ such that  $s_1-s_j \geq R$ and $t \leq T_0$ we have
\begin{eqnarray*}
\max_{i \in A_t} \widehat{s}_{i,t} - \widehat{s}_{j,t} & \geq & \widehat{s}_{1,t} - \widehat{s}_{j,t} \\
& \geq & s_1 - s_{j} -  \frac{2n}{n-1} \sqrt{ \frac{ \log(4 n t^2/\delta)}{2t}  } \\
& = &  \frac{ \Delta_{1,j}([n]) }{n-1} - \frac{2n}{n-1} \sqrt{ \frac{ \log(4 n t^2/\delta)}{2t}  } \  .
\end{eqnarray*}
If $\tau_j$ is the first time $t$ that the right hand side of the above is greater than or equal to $\frac{2n}{n-1 } \sqrt{ \frac{  \log(4 n t^2/\delta)}{2t}  }$ then
\begin{align*}
\tau_j &\leq \frac{32 n^2 }{\Delta_{1,j}^2([n]) } \log\left( \frac{32 n^3/ \delta}{\Delta_{1,j}^2([n]) }  \right) , 
%\tau_j &\leq \frac{32(n-1)^2 }{\Delta_{1,j}^2([n]) } \log\left( \frac{32 n^3/ \delta}{\Delta_{1,j}^2([n]) }  \right) , 
\end{align*}
since  for all positive $a,b,t$ with $a/b \geq e$ we have $t \geq  \frac{2\log(a/b)}{b} \implies b \geq \frac{ \log(a t)}{t}  $. Thus, any $j$ with $\frac{ \Delta_{1,j}([n]) }{n-1} = s_1-s_j \geq R$ has $\tau_j \leq T_0$ which implies that any $i \in A_t$ for $t > T_0$ has $s_1-s_i \leq R$.\\

\textbf{Step 2: For all $t$, $1 \in A_t$.}\\
We showed above that the Borda reduction will never remove the best arm from $A_t$. We now show that the sparse-structured discard condition will not remove the best arm. At any time $t > T_0$, let $i \in [n] \setminus 1$ be arbitrary and let $\widehat{\Omega}_i = \displaystyle\arg \max_{ \Omega \subset [n] : |\Omega| = k }  \widehat{\nabla}_{i,1,t}(\Omega)$ and ${\Omega_i} = \displaystyle\arg \max_{ \Omega \subset [n] : |\Omega| = k }  \nabla_{i,1}( \Omega )$. Note that for any $\Omega \subset [n]$ we have $\nabla_{i,1}(\Omega) = \nabla_{1,i}(\Omega)$ but $\Delta_{i,1}(\Omega) = -\Delta_{1,i}(\Omega)$ and
\begin{align*} 
\widehat{\Delta}_{i,1,t}( \widehat{\Omega}_i ) & \leq {\Delta}_{i,1}( \widehat{\Omega}_i ) + C_{t}( \widehat{\Omega}_i ) \\
&= {\Delta}_{i,1}( \widehat{\Omega}_i ) - {\Delta}_{i,1}( {\Omega_i} ) + {\Delta}_{i,1}( {\Omega_i} )  + C_{t}( \widehat{\Omega}_i ) \\
&= \left( \sum_{\omega \in \widehat{\Omega}_i} ( p_{i,\omega} - p_{1,\omega} ) \right) - \left( \sum_{\omega \in \Omega_i} ( p_{i,\omega} - p_{1,\omega} ) \right) -{\Delta}_{1,i}( {\Omega_i} )  + C_{t}( \widehat{\Omega}_i ) \\
&\leq -\left( \sum_{\omega \in \Omega_i \setminus \widehat{\Omega}_i} ( p_{i,\omega} - p_{1,\omega} ) \right) - \frac{2}{3} {\Delta}_{1,i}( {\Omega_i} )  + C_{t}( \widehat{\Omega}_i ) 
\end{align*}
since $ \left( \sum_{\omega \in \widehat{\Omega}_i \setminus \Omega_i} ( p_{i,\omega} - p_{1,\omega} ) \right) \leq \nabla_{1,i}\left( \widehat{\Omega}_i \setminus \Omega_i\right) \leq \frac{1}{3} {\Delta}_{1,i}( {\Omega_i} ) $ by the conditions of the theorem. Continuing,  
\begin{align*}
\widehat{\Delta}_{i,1,t}( \widehat{\Omega}_i ) &\leq -\left( \sum_{\omega \in \Omega_i \setminus \widehat{\Omega}_i} ( p_{i,\omega} - p_{1,\omega} ) \right) - \frac{2}{3} {\Delta}_{1,i}( {\Omega_i} )  + C_{t}( \widehat{\Omega}_i )  \\
& \leq  \left( \sum_{\omega \in \Omega_i \setminus \widehat{\Omega}_i} | \widehat{p}_{i,\omega,t} - \widehat{p}_{1,\omega,t} | \right) - \frac{2}{3} {\Delta}_{1,i}( {\Omega_i} )  + C_{t}( \widehat{\Omega}_i ) +  C_{t}( \Omega_i \setminus \widehat{\Omega}_i ) \\
&\leq \left( \sum_{\omega \in  \widehat{\Omega}_i \setminus \Omega_i } | \widehat{p}_{i,\omega,t} - \widehat{p}_{1,\omega,t} | \right) - \frac{2}{3} {\Delta}_{1,i}( {\Omega_i} )  + C_{t}( \widehat{\Omega}_i ) +  C_{t}( \Omega_i \setminus \widehat{\Omega}_i ) \\
&\leq \left( \sum_{\omega \in \widehat{\Omega}_i \setminus \Omega_i} | p_{i,\omega} - p_{1,\omega} | \right) - \frac{2}{3}{\Delta}_{1,i}( {\Omega_i} ) + C_{t}( \widehat{\Omega}_i )  +  C_{t}( \Omega_i \setminus \widehat{\Omega}_i )   +  C_{t}( \widehat{\Omega}_i \setminus \Omega_i  ) \\
&\leq - \frac{1}{3} {\Delta}_{1,i}( {\Omega_i} ) + C_{t}( \widehat{\Omega}_i )  +  C_{t}( \Omega_i \setminus \widehat{\Omega}_i )   +  C_{t}( \widehat{\Omega}_i \setminus \Omega_i  ) \\
&\leq  3 \max_{\Omega \subset [n]: |\Omega| \leq k} C_{t}(  \Omega  )  =  6(1 +  k )  C_t
\end{align*}
where the third inequality follows from the fact that $\widehat{\nabla}_{i,1,t}\left( \Omega_i \setminus  \widehat{\Omega}_i \right) \leq \widehat{\nabla}_{i,1,t}\left( \widehat{\Omega}_i \setminus \Omega_i \right)$ by definition, and the second-to-last line follows again by the same theorem condition used above. Thus, combining both steps one and two, we have that $1 \in A_t$ for all $t$. \\

\textbf{Step 3 : Sample Complexity}\\
At any time $t > T_0$, let $j \in [n] \setminus 1$ be arbitrary and let $\widehat{\Omega}_i = \displaystyle\arg \max_{ \Omega \subset [n] : |\Omega| = k }  \widehat{\nabla}_{1,j,t}(\Omega)$ and ${\Omega_i} = \displaystyle\arg \max_{ \Omega \subset [n] : |\Omega| = k }  \nabla_{1,j}( \Omega )$. We begin with
\begin{align*}
\hspace{1in}\hspace{-1in}\max_{i \in [n] \setminus j} \widehat{\Delta}_{i,j,t}\left(   \widehat{\Omega}_i \right) &\geq \widehat{\Delta}_{1,j,t}( \widehat{\Omega}_i )   \\
&\geq{\Delta}_{1,j}( \widehat{\Omega}_i ) - C_t( \widehat{\Omega}_i )  \\
&\geq{\Delta}_{1,j}( \widehat{\Omega}_i ) - \Delta_{1,j}(\Omega_i) + \Delta_{1,j}(\Omega_i) - C_t( \widehat{\Omega}_i )  \\
&=  \left( \sum_{\omega \in \widehat{\Omega}} ( p_{1,\omega} - p_{j,\omega} ) \right) - \left( \sum_{\omega \in \Omega_i} ( p_{1,\omega} - p_{j,\omega} ) \right)  + \Delta_{1,j}(\Omega_i) - C_t( \widehat{\Omega}_i ) \\
&\geq   - \left( \sum_{\omega \in \Omega_i \setminus \widehat{\Omega}} ( p_{i,\omega} - p_{1,\omega} ) \right)  + \frac{2}{3}\Delta_{1,j}(\Omega_i) - C_t( \widehat{\Omega}_i )  \\
&\geq  - \left( \sum_{\omega \in \Omega_i \setminus \widehat{\Omega}} | \widehat{p}_{i,\omega,t} - \widehat{p}_{1,\omega,t} | \right)  + \frac{2}{3}\Delta_{1,j}(\Omega_i)  - C_t( \widehat{\Omega}_i ) -  C_{t}( \Omega_i \setminus \widehat{\Omega}_i )  \\
&\geq  - \left( \sum_{\omega \in  \widehat{\Omega}_i \setminus \Omega_i } | \widehat{p}_{i,\omega,t} - \widehat{p}_{1,\omega,t} | \right)  + \frac{2}{3}\Delta_{1,j}(\Omega_i)   - C_t( \widehat{\Omega}_i ) -  C_{t}( \Omega_i \setminus \widehat{\Omega}_i ) \\
&\geq  - \left( \sum_{\omega \in  \widehat{\Omega}_i \setminus \Omega_i } | {p}_{i,\omega} - {p}_{1,\omega} | \right)  + \frac{2}{3}\Delta_{1,j}(\Omega_i)   - C_t( \widehat{\Omega}_i ) -  C_{t}( \Omega_i \setminus \widehat{\Omega}_i ) -  C_{t}(  \widehat{\Omega}_i \setminus \Omega_i ) \\
&\geq \frac{1}{3}{\Delta}_{1,j}( {\Omega_i} )  - 3 \max_{\Omega \subset [n]: |\Omega| \leq k} C_{t}(\Omega)  = \frac{1}{3} {\Delta}_{1,j}( {\Omega_i} )  - 6 (1 +  k )  C_t
\end{align*}
by a series of steps as analogous to those in Step 2. If $\tau_j$ is the first time $t > T_0$ such that the right hand side is greater than or equal to $6 (1 +  k )  C_t$, the point at which $j$ would be removed, we have that 
%\begin{align*}
% \frac{ {\Delta}_{1,j}^2( {\Omega_i} ) }{ 1152 n (k+1)^2}  \geq  \frac{\log(2 n t / \delta)}{t} 
%\end{align*}
%where $b =  \frac{ {\Delta}_{1,j}( {\Omega_i} ) }{ 1152 n (k+1)^2} $ and $a = 2 n / \delta$ so that $\frac{2}{b}= \frac{ 2304 n (k+1)^2}{ {\Delta}_{1,j}( {\Omega_i} ) } $
\begin{align*}
\tau_j \leq \frac{ 20736 n (k+1)^2}{ {\Delta}_{1,j}^2( {\Omega_i} ) } \log\left(  \frac{ 20736 n^2 (k+1)^2}{ {\Delta}_{1,j}^2( {\Omega_i} ) \, \delta}   \right)
\end{align*}
%\begin{align*}
%\tau_j \leq \frac{288 n (k+2) }{ \Delta_{1,j}^2( \Omega_i ) } \log\left( \frac{288 n^{k+2} (k+2)  /\delta }{ \Delta_{1,j}^2( \Omega_i )}  \right)
%\end{align*}
%using the fact that $\sqrt{a} + a/3 \leq ?$ 
using the same inequality as above in Step 2. Combining steps one and three we have that the total number of samples taken is bounded by
\begin{align*}
\sum_{j > 1} \min\left\{  \max\left\{ T_0,  \frac{ 20736 n (k+1)^2}{ {\Delta}_{1,j}^2( {\Omega_i} ) } \log\left(  \frac{ 20736 n^2 (k+1)^2}{ {\Delta}_{1,j}^2( {\Omega_i} ) \, \delta}   \right)  \right\} , \frac{32 n^2 }{\Delta_{1,j}^2([n]) } \log\left( \frac{32 n^3/ \delta}{\Delta_{1,j}^2([n]) }  \right)  \right\}
\end{align*}
with probability at least $1- 3\delta$. The result follows from recalling that $\frac{ \Delta_{1,j}(\Omega_i) }{n-1} = s_1-s_j$ and noticing that $\frac{n}{n-1} \leq 2$ for $n \geq 2$.
\end{proof}
\end{document}